\def\year{2018}\relax
\documentclass[letterpaper]{article} 
\usepackage{aaai18}  
\usepackage{times}  
\usepackage{helvet}  
\usepackage{courier}  
\usepackage{url}  
\usepackage{graphicx}  
\usepackage{booktabs} 
\usepackage{amsthm}
\usepackage{amsmath}
\usepackage{amssymb}
\usepackage{physics}
\usepackage{xcolor}
\usepackage{multirow}

\newcommand{\R}{\mathbb{R}}

\newcommand{\incite}[1]{\citeauthor{#1} \shortcite{#1}}

\newtheorem{proposition}{Proposition}
\theoremstyle{definition}
\newtheorem{definition}{Definition}

\begin{document}
%
\title{Mitigating Unwanted Biases with Adversarial Learning}
\author{Brian Hu Zhang \\
Stanford University \\
Stanford, CA \\
bhz@stanford.edu
\And
Blake Lemoine \\
Google \\
Mountain View, CA \\
lemoine@google.com
\And
Margaret Mitchell \\
Google \\
Mountain View, CA \\
mmitchellai@google.com
}
\maketitle
\begin{abstract}
Machine learning is a tool for building models that accurately represent input training data. When undesired biases concerning demographic groups are in the training data, well-trained models will reflect those biases. We present a framework for mitigating such biases by including a variable for the group of interest and simultaneously learning a predictor and an adversary. The input to the network X, here text or census data, produces a prediction Y, such as an analogy completion or income bracket, while the adversary tries to model a protected variable Z, here gender or zip code.

The objective is to maximize the predictor’s ability to predict Y while minimizing the adversary's ability to predict Z. Applied to analogy completion, this method results in accurate predictions that exhibit less evidence of stereotyping Z. When applied to a classification task using the UCI Adult (Census) Dataset, it results in a predictive model that does not lose much accuracy while achieving very close to equality of odds (Hardt, et al., 2016). The method is flexible and applicable to multiple definitions of fairness as well as a wide range of gradient-based learning models, including both regression and classification tasks. 
\end{abstract}

\maketitle

\section{Introduction}
Machine learning leverages data to build models capable of assessing the labels and properties of novel data.  Unfortunately, the available training data frequently contains biases with respect to things that we would rather not use for decision making.  Machine learning builds models faithful to training data and can lead to perpetuating these undesirable biases.  For example, systems designed to predict creditworthiness and systems designed to perform analogy completion have been demonstrated to be biased against racial minorities and women respectively.  Ideally we would be able to build a model which captures exactly those generalizations from the data which are useful for performing some task which are not discriminatory in a way which the people building those models consider unfair.

Work on training machine learning systems that output fair decisions has defined several useful measurements for {\it fairness}:  Demographic Parity, Equality of Odds, and Equality of Opportunity.  These can be imposed as constraints or incorporated into a loss function in order to mitigate disproportional outcomes in the system's output predictions regarding a protected demographic, such as sex.

In this paper, we examine these fairness measures in the context of {\it adversarial debiasing}.  We consider supervised deep learning tasks in which the task is to predict an output variable $Y$ given an input variable $X$, while remaining unbiased with respect to some variable $Z$.  We refer to $Z$ as the {\it protected variable}.  For these learning systems, the predictor $\hat Y = f(X)$ can be constructed as (input, output, protected) tuples $(X, Y, Z)$.
The predictor $f(X)$ is usually given access to the protected variable $Z$, though this is not strictly necessary.  This construction allows the determination of which types of bias are considered undesirable for a particular application to be chosen through the specification of the protected variable.

We speak to the concept of {\it mitigating bias} using the known term {\it debiasing}\footnote{Note that ``debias'' may not be quite the right word, as all bias is not necessarily removed.}, following definitions provided by \incite{hardt2016equality} and refined by \incite{beutel2017data}. 

\begin{definition}{{\sc Demographic Parity.}} 
A predictor $\hat Y$ satisfies {\it demographic parity} if $\hat Y$ and $Z$ are independent. 
\end{definition}

 This means that $P(\hat Y = \hat y)$ is equal for all values of the protected variable $Z$: $P(\hat Y = \hat y) = P(\hat Y = \hat y |Z = z)$.

\begin{definition}{{\sc Equality of Odds.}} 
A predictor $\hat Y$ satisfies {\it equality of odds} if $\hat Y$ and $Z$ are conditionally independent given $Y$. 
\end{definition}

This means that, for all possible values of the true label $Y$, $P(\hat Y = \hat y)$ is the same for all values of the protected variable: $P(\hat Y = \hat y|Y = y) = P(\hat Y = \hat y|Z = z, Y = y)$

\begin{definition}{{\sc Equality of Opportunity.}}  
If the output variable $Y$ is discrete, a predictor $\hat Y$ satisfies {\it equality of opportunity} with respect to a class $y$ if $\hat Y$ and $Z$ are independent conditioned on $Y=y$.
\end{definition}

This means that, for a {\it particular} value of the true label $Y$, $P(\hat Y = \hat y)$ is the same for all values of the protected variable: $P(\hat Y = \hat y|Y = y) = P(\hat Y = \hat y|Z = z, Y = y)$

We present an adversarial technique for achieving whichever one of these definitions is desired.\footnote{Achieving equality of odds and demographic parity are generally incongruent goals.  See also \incite{kleinberg2016inherent} for incongruency between calibration and equalized odds.} 
A predictor $f$ will be trained to model $Y$ as accurately as possible while satisfying one of the above equality constraints.  Demographic parity will be achieved by introducing an adversary $g$ which will attempt to predict a value for $Z$ from $\hat Y$.  The gradient of $g$ will then be incorporated into the weight update rule of $f$ so as to reduce the amount of information about $Z$ transmitted through $\hat Y$.  Equality of odds will be achieved by also giving $g$ access to the true label $Y$, thereby limiting any information about $Z$ which $\hat Y$ contains beyond the information already contained in $Y$. 

We consider the case where the protected variable is a discrete feature present in the training set as well as the case in which the protected variable must be inferred from latent semantics (in particular, gender from word embeddings).  In order to accomplish the latter we adapt a technique presented by \incite{bolukbasi2016man} to define a subspace capturing the semantics of the protected variable, and then train a model to perform a word analogies task accurately, while unbiased on this protected variable.
A consequence of this technique is that the network learns {\it ``debiased"} embeddings, embeddings that have the semantics of the protected variable removed. These embeddings are still able to perform the analogy task well, but are better at avoiding problematic examples such as those shown in \incite{bolukbasi2016man}.

Results on the UCI Adult Dataset demonstrate the technique we introduce allows us to train a model that  achieves equality of odds to within 1\% on both protected groups. 

We also compare with the related previous work of  \incite{beutel2017data}, and find we are able to better equalize the differences between the two groups, measured by both False Positive Rate and False Negative Rate (1 - True Positive Rate), although note that the previous work performs better overall for False Negative Rate.

We provide some discussion on caveats pertaining to this approach, difficulties in training these models that are shared by many adversarial approaches, as well as some discussion on  difficulties that the fairness constraints introduce.  

\section{Related Work}

There has been significant work done in the area of debiasing various specific types of data or predictor.

{\it Debiasing word embeddings}: \incite{bolukbasi2016man} devises a method to remove gender bias from word embeddings. The method relies on a lot of human input; namely, it needs a large ``training set'' of gender-specific words. 

{\it Simple models}: \incite{lum2016statistical} demonstrate that removing the protected variable from the training data fails to yield a debiased model (since other variables can be highly correlated with the protected variable), and devise a method for learning fair predictive models in cases when the learning model is simple (e.g. linear regression). \incite{hardt2016equality} discuss the shortcomings of focusing solely on {\sc demographic parity}, present alternate definitions of fairness, and devise a method for deriving an unbiased predictor from a biased one, in cases when both the output variable and the protected variable are discrete.

{\it Adversarial training}: \incite{goodfellow2014generative} pioneered the technique of using multiple networks with competing goals to force the first network to ``deceive'' the second network, applying this method to the problem of creating real-life-like pictures. \incite{beutel2017data} apply an adversarial training method to achieve {\sc equality of opportunity} in cases when the output variable is discrete. They also discuss the ability of the adversary to be powerful enough to enforce a fairness constraint even when it has access to a very small training sample.

\section{Adversarial Debiasing} \label{section:adversarial-debiasing}
\begin{figure}
    \centering
    \includegraphics[width=\columnwidth]{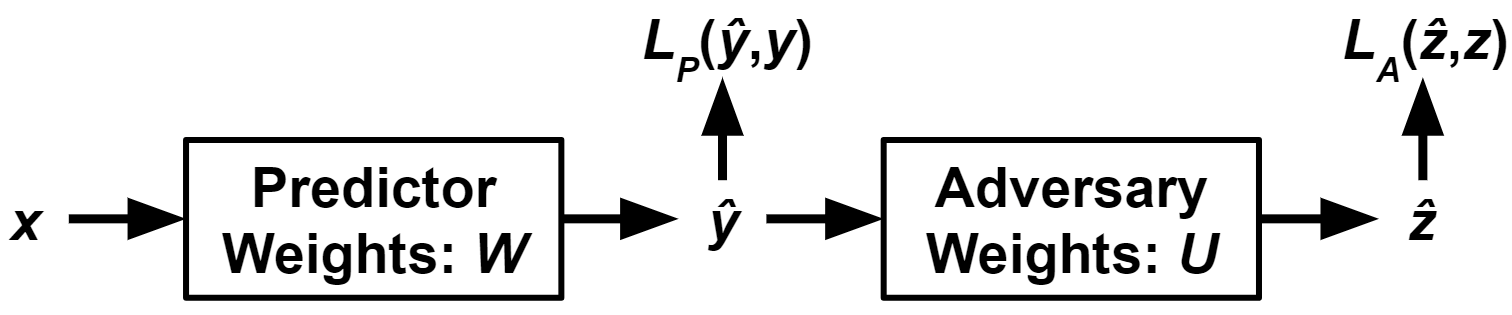}
    \caption{The architecture of the adversarial network.}
    \label{fig:arch}
\end{figure}
We begin with a model, which we call the {\it predictor}, trained to accomplish the task of predicting $Y$ given $X$. As in Figure \ref{fig:arch}, we assume that the model is trained by attempting to modify weights $W$ to minimize some loss $L_P(\hat y, y)$, 
using a gradient-based method such as stochastic gradient descent.

The output layer of the predictor is then used as an input to another network called the {\it adversary} which attempts to predict $Z$. This is part of the network corresponds to the {\it discriminator} in a typical GAN \cite{goodfellow2014generative}. We will suppose the adversary has loss term $L_A(\hat z, z)$ and weights $U$.
Depending on the definition of fairness being achieved, the adversary may have other inputs. 
\begin{itemize}
\item For {\sc Demographic Parity}, the adversary gets the predicted label $\hat Y$. Intuitively, this allows the adversary to try to predict the protected variable using nothing but the predicted label. The goal of the predictor is to prevent the adversary from doing this. 

\item For {\sc Equality of Odds}, the adversary gets $\hat Y$ and the true label $Y$. 
\item For {\sc Equality of Opportunity} on a given class $y$, we can restrict the training set of the adversary to training examples where $Y = y$.\footnote{This last technique of restricting the training set is discussed at length by \incite{beutel2017data}, so we only mention it here.}
\end{itemize}

In order for gradients to propagate correctly, $\hat Y$ above refers to the output layer of the network, not to the discrete prediction; for example, for a classification problem, $\hat Y$ could refer to the output of the softmax layer. 

We update $U$ 
to minimize $L_A$ at each training time step, according to the gradient $\nabla_U L_A$.  We modify $W$ according to the expression:
\begin{equation} \label{eq:grads}
    \nabla_W L_P - \text{proj}_{\nabla_W L_A} \nabla_W L_P - \alpha \nabla_W L_A 
\end{equation}
where $\alpha$ is a tuneable hyperparameter that
can vary at each time step and we define $\text{proj}_v x = 0$ if $v = 0$.

\begin{figure}
    \centering
    \includegraphics[width=\columnwidth]{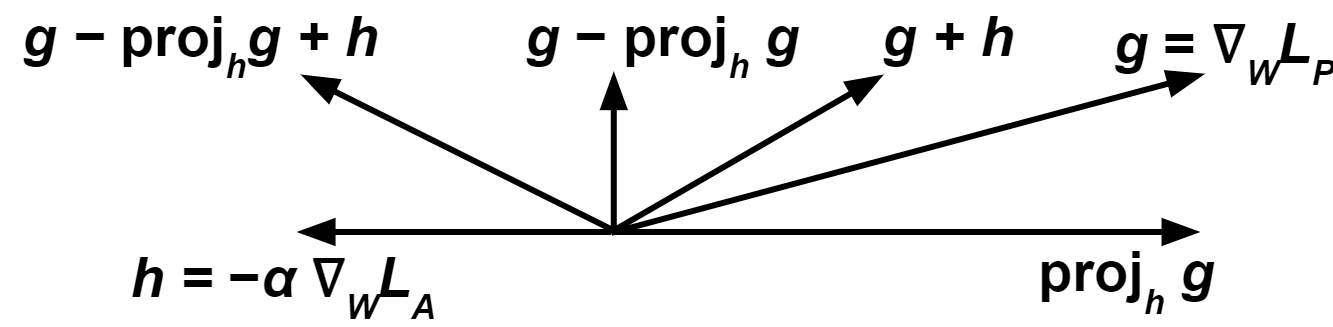}
    \caption{Diagram illustrating the gradients in Eqn. \ref{eq:grads} and the relevance of the projection term $\text{proj}_h g$. Without the projection term, in the pictured scenario, the predictor would move in the direction labelled $g + h$ in the diagram, which actually \textit{helps} the adversary. With the projection term, the predictor will never move in a direction that helps the adversary.}
    \label{fig:grads}
\end{figure}

The middle term $\text{proj}_{\nabla_W L_A} \nabla_W L_P$ prevents the predictor from moving in a direction that helps the adversary decrease its loss while the last term, $\alpha \nabla_W L_A$, attempts to increase the adversary's loss. Without the projection term, it is possible for the predictor to end up helping the adversary (see Fig. \ref{fig:grads}).  Without the last term, the predictor will never try to {\it hurt} the adversary, and, due to the stochastic nature of many gradient-based methods, will likely end up helping the adversary anyway. The result is that when training is completed the desired definition of equality should be satisfied.

Notice that our definitions and method make no assumptions about the nature of the output and protected variables: in particular, they work with both regression and classification models, as well as with both discrete and continuous protected variables.

\section{Properties}

We note several properties of the above method that we believe distinguish it from past work.

\begin{enumerate}
\item{\bf Generality:} The above method can be used to enforce {\sc demographic parity}, {\sc equality of odds}, or {\sc equality of opportunity} as described in \incite{hardt2016equality}. Further, it applies without modification to the cases when the output variable and/or protected variable are continuous instead of discrete.

\item{\bf Model-agnostic:}  The adversarial approach described can be applied regardless of how simple or complex the predictor's model is, as long as the model is trained using a gradient-based method, as many modern learning models are. Further, as we will discuss later, at least in some situations, we suggest that the adversary does not need to be nearly as complex as the predictor---a simple adversary can be used with a complex predictor.  

\item{\bf Optimality:}  Under certain conditions, we show that if the predictor converges, it must converge to a model that satisfies the desired fairness definition. Since the predictor also attempts to decrease the prediction loss $L_P$, the predictor should still perform well on the target task.
\end{enumerate}

\section{Theoretical Guarantees} \label{section:theory}
\begin{proposition}
Let the predictor, the adversary, and their weights $W$,  $U$ be defined according to Section \ref{section:adversarial-debiasing} 
Let $L_A(W, U)$ be the adversary's loss, convex in $U$, concave in $W$,\footnote{We understand that these assumptions are not satisfied in most use cases involving neural networks; however, as with most theoretical analyses of machine learning models (see, for example, \incite{goodfellow2014generative} or \incite{kingma2014adam}; the former makes even stronger assumptions), assumptions of concavity are necessary for any proofs to work} and continuously differentiable everywhere.

Suppose that: 
\begin{enumerate}
    \item When the predictor's weights are $W_0$, the predictor gives the same output $\hat Y$ regardless of input $X$. (For example, when $W_0 = 0$). 
    \item There are some weights $U_0$ that minimize $L_A$ when the weights for $\hat Y$ have no effect on the output: For all $W$, $L_A(W, U_0) = \min_U L_A(W_0, U)$. 
    \item Predictor and adversary converge to $W^*$ and $U^*$ respectively.
\end{enumerate} 

Then, $L_A(W^*, U^*) = L_A(W^*, U_0)$. That is, the adversary gains no advantage from using the weights for $\hat Y$. 

\end{proposition}

\begin{proof}
Since the adversary converges, $L_A(W^*, U^*) \le L_A(W^*, U_0)$: otherwise, since $L_A$ is convex in $U$, the adversary's weights would move toward $U_0$.  In other words, the adversary's minimum is the point at which the adversary gains an advantage from using $\hat Y$.  Similarly, since the predictor converges, $L_A(W^*, U^*) \ge L_A(W_0, U^*)$: Otherwise, the predictor would be able to increase the adversary's loss by moving toward $W_0$, and the projection term and negative weight on $\nabla_W L_A$ in Eqn. \ref{eq:grads} would push the predictor to move towards $0$. Then: 
\begin{alignat*}{3}
    L_A(W^*, U_0) &\ge L_A(W^*, U^*)  &\qq{(as stated above)}
    \\&\ge L_A(W_0, U^*) &\qq{(as stated above)}
    \\&\ge L_A(W_0, U_0) &\qq{(by definition of $U_0$)}
    \\&= L_A(W^*, U_0) &\qq{(by definition of $U_0$)}
\end{alignat*}
so we must have $L_A(W^*, U^*) = L_A(W^*, U_0)$.
\end{proof}

Note that, in this proof, the adversary can be operating in a few different ways, as long as it is given $\hat Y$ as one of its inputs; for example, for demographic parity, it could be given only $\hat Y$; for equality of odds, it can be given both $\hat Y$ and $Y$.

We will show in the next propositions that the adversary gaining no advantage from information about $\hat Y$ is exactly the condition needed to guarantee that desired definitions of equality are satisfied.

\begin{proposition} \label{prop:parity}
Let the training data be comprised of triples $(X, Y, Z)$ drawn according to some distribution $D$. Suppose:
\begin{enumerate}
    \item The protected variable $Z$ is discrete.
    \item The adversary is trained for {\sc demographic parity}; i.e. the adversary is given only the prediction $\hat y$. 
    \item The adversary is strong enough that, at convergence, it has learned a randomized function A that minimizes the cross-entropy loss $\mathbb E_{(x, y, z) \sim D}[-\log P(A(\hat y) = z)]$; i.e. the adversary in fact achieves the optimal accuracy with which you can predict $Z$ from $\hat Y$
    \item The predictor completely fools the adversary; in particular, the adversary achieves loss $H(Z)$, the entropy of $Z$.
\end{enumerate}
Then the predictor satisfies {\sc demographic parity}; i.e., $\hat Y \perp Z$.
\end{proposition}
\begin{proof} 
Notice that if the adversary draws $A(\hat y)$ according to the distribution $Z|\hat Y = \hat y$, then its loss is exactly the conditional entropy
\begin{alignat*}{2}
H(Z|\hat Y) &= \mathbb E[-\log P(Z=z | \hat Y = \hat y)] \\
&= \mathbb E[-\log P(A(\hat y) = z | \hat Y = \hat y)]
\end{alignat*}
where the expectation is taken over $(x, y, z) \sim D$.
Now suppose for contradiction that $\hat Y$ is dependent on $Z$. Then $H(Z|\hat Y) < H(Z)$, so the adversary can achieve loss less than $H(Z)$, contradicting assumption (4).
\end{proof}

\begin{proposition} \label{prop:eq-odds}
If assumptions (2)-(4) above are replaced with the analogous equality of odds assumptions; in particular, that the adversary is given $\hat y$ and $y$, and the adversary cannot achieve loss better than $H(Z|Y)$~then the predictor will satisfy {\sc Equality of Odds}; i.e., $(\hat Y \perp Z) | Y$
\end{proposition}
\begin{proof} 
Analogous to the above. Notice that if the adversary draws $A(\hat y, y) \sim (Z|\hat Y = \hat y, Y = y)$, then its loss is exactly the conditional entropy 
\begin{alignat*}{2}
H(Z|\hat Y,Y) &= \mathbb E[-\log P(Z=z | \hat Y=\hat y,Y=y)] \\
&= \mathbb E[-\log P(A(\hat y) = z | \hat Y=\hat y,Y=y)]
\end{alignat*}
where the expectation is again taken over $(x, y, z) \sim D$. But if $\hat Y$ is conditionally dependent on $Z$ given $Y$, then $H(Z|\hat Y, Y) < H(Z|Y)$, so the adversary can achieve loss less than $H(Z|Y)$.
\end{proof}

Note that Propositions \ref{prop:parity} and \ref{prop:eq-odds} work analogously in the case of continuous $Y$ and $Z$, with the probability mass function $P$ replaced with the probability density function $p$, and the discrete entropy $H$ replaced by the differential entropy $h(X) = \mathbb E[-\log p(x)]$, since the relevant property ($h(A) = h(A|B)$ iff $A \perp B$) holds for differential entropy as well. They also work analogously when the adversary $A$ is restricted to a limited set of predictors. 

For example, an adversary using least-squares regression trying to enforce equality of odds can be thought of as one that outputs $A(\hat y, y) \sim N(\mu(\hat y, y), \sigma^2)$ where $\mu(\hat y, y)$ is the output of the regressor, and $\sigma^2 > 0$ is a fixed constant. 
Note now that the differential entropy $h(Z|\hat Y, Y) = \mathbb{E}
[-\log p(z|\hat y, y)]$ is nothing more than the expected log-likelihood, and so the function $\mu$ that minimizes this quantity is the optimal least-squares regressor. Thus, for example, if we restrict $\mu$ to be a linear function of $(\hat y, y)$, and the other conditions of Proposition \ref{prop:eq-odds} hold, then an analogous argument to the above propositions shows that $\hat Y$ has no linear relationship with $Z$ after conditioning on $Y$. 

These claims together illustrate that a sufficiently powerful adversary trained on a sufficiently large training set can indeed accurately enforce the demographic parity or equality of odds constraints on the predictor, if the adversary and predictor converge. Guaranteed convergence is harder to achieve, both in theory and practice. In the practical scenarios below we discuss methods to encourage the training algorithm to converge, as well as reasonable choices of the adversary model that are both powerful and easy to train.

\section{Experiments} \label{section:experiments}

All models were trained using the Adam optimizer \cite{kingma2014adam} for both predictor and adversary.

\subsection{Toy Scenario}

We generate a training sample ${(x^{(i)},y^{(i)},z^{(i)})}_{i=1}^{n}$ (where $z$ is the protected variable) as follows. For each $i$, let $r \in {0,1}$ be picked uniformly at random, and let $v \sim N(r_{i},1)$. Let $u,w \sim N(v_{i},1)$ vary independently. Then $x^{(i)}=(r,u),y^{(i)}=[w>0],z^{(i)}=r$. (where $[~]$ denotes an indicator function). Intuitively, the variable that we are trying to predict, $y$, depends directly on $v$ and $r$. We are given as inputs the protected variable $r$, and a noisy measurement of $v$.  The end goal would be to train a model that predicts $y$ while being unbiased on $r$, effectively removing the direct signal for $r$ from the learned model. 

If one trains generically a logistic regression model to predict $y$ given $x$, it outputs something like $y=\sigma(0.7u+0.7r)$, which is a reasonable model, but heavily incorporates the protected variable $r$. To debias, We now train a model that achieves {\sc demographic parity}. Note that removing the variable $r$ from the training data is insuffucient for debiasing: the model will still learn to use $u$ to predict $y$, and $u$ is correlated with $r$. If we use the described technique and add in another logistic model that tries to predict $z$ given $y$, we find that the predictor model outputs something like $y=\sigma(0.6u-0.6r+0.6)$. Notice that not only is $r$ not included with a positive weight anymore, the model actually learns to use a negative weight on $r$ in order to balance out the effect of $r$ on $u$  Notice that $u-r \sim N(0,2)$; i.e., it is not dependent on $r$, so we have successfully trained a model to predict $y$ independently of $r$.

\subsection{Word Embeddings}

We train a model to perform the analogy task (i.e., fill in the blank: \texttt{man : woman :: he : ?}).

It is known that word embeddings reflect or amplify problematic biases from the data they are trained on, for example, gender \cite{bolukbasi2016man}. We seek to train a model that can still solve analogies well, but is less prone to these gender biases. We first calculate a ``gender direction'' $g$ using a method based on \incite{bolukbasi2016man} which gives a method for defining the protected variable. We will use this technique in the context of defining gender for word embeddings, but, as discussed in \incite{bolukbasi2016man}, the technique generalizes to other protected variables and other forms of embeddings. Following \incite{bolukbasi2016man}, we pick 10 (male, female) word pairs, and define the and define the {\it bias subspace} to be the space spanned by the top $k$ principal components of the differences, where $k$ is a tuneable parameter. In our experiments, we find that $k = 1$ gives reasonable results, so we did not experiment further.

We use embeddings trained from Wikipedia to generate input data from the Google analogy data set \cite{mikolov2013distributed}.  For each analogy in the dataset, we let $x=(x_{1},x_{2},x_{3}) \in \R^{3d}$ comprise the word vectors for the first three words, $y$ be the word vector of the fourth word, and $z$ be $\text{proj}_{g}y$.  It is worth noting that these word vectors computed from the original embeddings are never updated nor is there projection onto the {\it bias subspace} and therefore the original word embeddings are never modified.  What is learned is a tranform from a biased embedding space to a debiased embedding space.

\begin{table}
\begin{center}
  \begin{tabular}{| l | c | l | c |}
    \hline
    \multicolumn{2}{|c|}{biased} & \multicolumn{2}{c|}{debiased} \\ \hline
    neighbor & similarity & neighbor & similarity \\
    \hline \hline
    nurse & 1.0121 & nurse & 0.7056 \\ \hline
    nanny & 0.9035 & obstetrician & 0.6861 \\ \hline
    fianc\'{e}e & 0.8700 & pediatrician & 0.6447 \\ \hline
    maid & 0.8674 & dentist & 0.6367 \\ \hline
    fianc\'{e} & 0.8617 & surgeon & 0.6303 \\ \hline
    mother & 0.8612 & physician & 0.6254 \\ \hline
    fiance & 0.8611 & cardiologist & 0.6088 \\ \hline
    dentist & 0.8569 & pharmacist & 0.6081 \\ \hline
    woman & 0.8564 & hospital & 0.5969 \\ \hline
  \end{tabular}
  \caption{Completions for \texttt{\small he : she :: doctor : ?}}
  \label{table:heshedoctor}
\end{center}
\end{table}

As a model, we use the following: let $v=x_{2}+x_{3}-x_{1}$, and output $\hat{y}=v-ww^{T}v$, where our model parameter is $w$. Intuitively, $v$ is the ``generic'' analogy vector as is commonly\footnote{see e.g. \incite{mikolov2013distributed}} 
used for the analogy task. If left to its own devices ({\it i.e.}, if not told to be unbiased on anything), the model should either learn $w=0$ or else learn $w$ as a useless vector. 

By contrast, if we add the adversarial discriminator network (here, simply $\hat{z}=w_{2}^{T}\hat{y}$), we expect the debiased prediction model to learn that $w$ should be something close to $g$ (or $-g$), so that the discriminator cannot predict $z=proj_{g}y$. Indeed, both of these expectations hold: Without debiasing, the trained vector $w$ is approximately a unit vector nearly perpendicular to $g: w^{T}g=0.08,||w||=0.82$; with debiasing, $w$ is approximately a unit vector pointing in a direction highly correlated with $g: w^{T}g=0.55,||w||=0.96$. Even after debiasing, gendered analogies such as \texttt{man : woman :: he : she} are still preserved; however, many biased analogies go away, suggesting that the adversarial training process was indeed successful. An example of the kinds of changes in analogy completions observed after debiasing are illustrated in Table \ref{table:heshedoctor}\footnote{The presence of {\it nurse} in the second position may seem worrying, but it should be noted that in this particular set of word embeddings, {\it nurse} is the nearest neighbor to {\it doctor}; no amount of debiasing will change this.}.

\subsection{UCI Adult Dataset}\label{subsec:uci}

\begin{table}
    \centering
    \small
    \begin{tabular}{@{ }l@{ }|@{\hspace{.25em}}l@{\hspace{.25em}}|@{\hspace{.4em}}p{16em}} 
    {\bf Feature} & {\bf Type} &  {\bf Description} \\\hline
age	& Cont & Age of the individual\\
capital\_gain& Cont& Capital gains recorded\\
capital\_loss& Cont& Capital losses recorded\\
education\_num& Cont& Highest education level (numerical form)\\
fnlwgt& Cont& \# of people census takers believe that observation represents\\
hours\_per\_week& Cont& Hours worked per week\\\hline
education& Cat& Highest level of education achieved\\
income& 	Cat& Whether individual makes $>$ \$50K annually   \\
marital\_status& Cat& Marital status\\
native\_country& Cat& Country of origin\\
occupation& Cat& Occupation \\
race& Cat& White, Asian-Pac-Islander, Amer-Indian-Eskimo, Other, Black\\
relationship& Cat& Wife, Own-child, Husband, Not-in-family, Other-relative, Unmarried\\
sex& Cat& Female, Male\\
workclass&  Cat& Employer type \\
 \end{tabular}
    \caption{Features in the UCI dataset per individual.  Features are either continuous (Cont) or Categorical (Cat).  Categorical features are converted to sparse tensors for the model.}
    \label{tab:uci}
\end{table}

To better align with the work in \incite{beutel2017data}, we attempt to enforce {\sc equality of odds} on a model for the task of predicting the income of a person -- in particular, predicting whether the income is $>\$50k$ -- given various attributes about the person, as made available in the UCI Adult dataset \cite{asuncion2007uci}.

Details on the features that the dataset provides are available in Table \ref{tab:uci}.  We use both categorical and continuous columns as given, with exception to the {\tt fnlwgt} feature, which we discard.  We convert the remaining columns into tensors where the categorical columns are sparse tensors, age is bucketized at boundaries $[18, 25, 30, 35, 40, 45, 50, 55, 60, 65]$, and the rest of the continuous columns are real-valued.

As discussed before, to enforce equality of odds, we give the adversary access to the true label $y$. The adversary will learn the relationship between $y$ and $z$ regardless of what the predictor does; further, if the predictor's predictions $\hat{y}$ give more information about $z$ than is already contained in $y$, the adversary will be able to improve its loss. Thus, the predictor, in attempting to fool the adversary, will move toward making sure that $\hat{y}$ does not give such additional information; in other words, toward equality of odds.

Our protected variable $z$ is a binary-valued variable for the two sexes annotated, {\sc male} and {\sc female}.

Our predictor model is straightforward logistic regression: $\hat{y}=\sigma(w_{1}\cdot x+b)$, where $\sigma$ is the sigmoid function.  Our adversary model takes the form of the following logistic-regression-like model:
\begin{alignat*}{2}
s=\sigma\qty((1+|c|)\sigma^{-1}(\hat{y}))\quad \hat{z}=w_{2}\cdot[s,sy,s(1-y)]+b
\end{alignat*}
where $c$ and $b$ are learnable scalars,  $w_{2}$ is a learnable vector, and $\sigma^{-1}$ is the inverse of the sigmoid function (logit function) $\sigma^{-1}(t)=\log t-\log (1-t)$. Intuitively, we want our adversary to be able to learn functions of the form $\hat{z}=f(y,[\hat{y}>0.5])$ (i.e. dependent only on the boolean predicted value $[\hat{y}>0.5]$), and thus enforce equality of odds. Here, the adversary would learn such a function by making $c$ extremely large. We add 1 to $|c|$ to make sure the adversary never tries to ignore $\hat{y}$ by setting $c=0$, which could be a difficult local minimum for the adversary to escape\footnote{This value added to $|c|$ is an adjustable hyperparameter; we found reasonable results using the value 1 and thus not feel the need to experiment further.}. This adversary is both general enough to be used whenever $y$ and $z$ are both discrete\footnote{If $y$ and $z$ are multi-class, then the sigmoid becomes a softmax, but everything else remains the same.}, and powerful enough that deviation from true equality of odds should cause the adversary to be able to decrease its loss.

Without tweaking, this algorithm ran into issues with local minima, and the resulting models were often closer to demographic parity than equality of odds.
We implemented a technique that helped: by increasing the hyperparameter $\alpha$ in Eqn. \ref{eq:grads} over time, the predictor had a much easier time learning to deceive the adversary and therefore more strictly enforce equality of odds. We set $\alpha=\sqrt{t}$ (where $t$ is the step counter), and to avoid divergence we set the predictor's step size to $\eta \propto 1/t$, so that $\alpha \eta \rightarrow 0$ as is preferred for stochastic gradient-based methods such as Adam. 

We train the model twice, once with debiasing and once without, and present side-by-side confusion matrices on the test set for income bracket with respect to the protected variable values {\it Male} and {\it Female}, shown in Table \ref{table:adult}, and we present the false positive rates (FPR) and false negative rates (FNR) in Table \ref{tab:false_rates}.  Note that false negative rate is equal to $1 -{}$ true positive rate, so the trade-offs are directly comparable to the $(x,y)$ values of an ROC curve.

\begin{table}[h!tb]
    \begin{center}
    \begin{tabular}{|c|c|c|c|c|c|c|}
    \hline
    \multicolumn{3}{|c|}{Without Debiasing} & \multicolumn{3}{|c|}{With Debiasing}   \\
    \hline
    \it Female & Pred 0 & Pred 1 & \it Female & Pred 0 & Pred 1  \\
    \hline
    True 0 & 4711 & 120 & True 0 & 4518 & 313  \\
    \hline
    True 1 & 265 & 325 & True 1 & 263 & 327  \\
    \hline
    \it Male & Pred 0 & Pred 1 & \it Male & Pred 0 & Pred 1  \\
    \hline
    True 0 & 6907 & 697 & True 0 & 7071 & 533  \\
    \hline
    True 1 & 1194 & 2062 & True 1 & 1416 & 1840  \\
    \hline
    \end{tabular}
    \caption{Confusion matrices on the UCI Adult dataset, with and without equality  of odds enforcement.}
    \label{table:adult}
    \end{center}
\end{table}

\begin{table}
\resizebox{\columnwidth}{!}{
\begin{tabular}{l|l|ll|ll|}
& & \multicolumn{2}{c|}{Female} & \multicolumn{2}{|c|}{Male} \\
 & & Without & With  & Without & With \\\hline
\multirow{2}{*}{\incite{beutel2017data}} & FPR & 0.1875 & 0.0308 & 0.1200 & 0.1778 \\
& FNR & 0.0651 & 0.0822 & 0.1828 & 0.1520 \\\hline
 
\multirow{2}{*}{Current work} & FPR & 0.0248 & 0.0647 & 0.0917 & 0.0701 \\
& FNR & 0.4492 & 0.4458 & 0.3667 & 0.4349 \\\hline
\end{tabular}}
\caption{False Positive Rate (FPR) and False Negative Rate (FNR) for income bracket predictions for the two sex subgroups, with and without adversarial debiasing.}\label{tab:false_rates}
\end{table}

We notice that debiasing has only a small effect on overall accuracy ($86.0\%$ vs $84.5\%$), and that the debiased model indeed (nearly) obeys equality of odds: as shown in Table 4, with debiasing, the FNR and FPR values are approximately equal across sex subgroups: $0.0647 \approx 0.0701$ and $0.4458 \approx 0.4349$.

Although the values don't exactly reach equality, neither difference is statistically significant: a two-proportion two-tail large sample $z$-test yields $p$-values 0.25 for $y=0$ and 0.62 for $y=1$.

\section{Conclusion}

In this work, we demonstrate a general and powerful method for training unbiased machine learning models. 
We state and prove theoretical guarantees for our method under reasonable assumptions, demonstrating in theory that the method can enforce the constraints that we claim, across multiple definitions of fairness, regardless of the complexity of the predictor's model, or the nature (discrete or continuous) of the predicted and protected variables in question. We apply the method in practice to two very different scenarios: a standard supervised learning task, and the task of debiasing word embeddings while still maintaining ability to perform a certain task (analogies).  We demonstrate in both cases the ability to train a model that is demonstrably less biased than the original one, and yet still performs extremely well on the task at hand. We discuss difficulties in getting these models to converge. We propose, in the common case of discrete output and protected variables, a simple adversary that is usable regardless of the complexity of the underlying model. 

\section{Future Work}

This process yields many questions that require further work to answer.
\begin{enumerate}
    \item The debiased word embeddings we have trained are still useful in analogies. Are they still useful in other, more complex tasks?
    \item The adversarial training method is hard to get right and often touchy, in that getting the hyperparameters wrong results in quick divergence of the algorithm. What ways can be used to stabilize training and ensure convergence, and thus ensure that the theoretical guarantees presented here can work?
    \item There is a body of existing work for image recognition using adversarial networks.  Image recognition in general can sometimes be subject to various biases such as being more or less successful at recognizing the faces of people of different races.  Can multiple adversaries be combined to create high accuracy image recognition systems which do not exhibit such biases?
    \item In general, do more complex predictors require more complex adversaries? It would appear that in the case of $y$ and $z$ discrete, a very simple adversary suffices no matter how complex the predictor. Does this also apply to continuous cases, or would a simple adversary be too easy to deceive for a complex predictor?
\end{enumerate}

\bibliographystyle{aaai}

\end{document}